\documentclass[a4paper]{llncs}

\usepackage[latin1]{inputenc}
\usepackage{algorithm}
\usepackage{algpseudocode}  
\usepackage{amsmath}
\usepackage{amssymb}
\usepackage{leqno}              
\usepackage{graphicx}
\usepackage{color}
\usepackage{multirow}
\usepackage{comment}
\usepackage{wrapfig}

\newcommand{\C}{{\mathcal C}}
\newcommand{\D}{{\mathcal D}}

\newcommand{\X}{{\mathcal X}}
\newcommand{\Y}{{\mathcal Y}}
\newcommand{\Z}{{\mathcal Z}}
\newcommand{\XH}{\hat\X}
\newcommand{\vars}{\operatorname{vars}}

\newcommand{\sol}{\operatorname{sol}}
\newcommand{\sizeof}[1]{|#1|}
\newcommand{\restrict}[2]{{#1}_{|#2}}
\newcommand{\propagate}{\operatorname{propagate}}

\newcommand{\pseudosubsubsec}[1]{\vspace{0.3em}\noindent\textbf{#1}.}

\def\CPP{\leavevmode\textrm{\hbox{C\hskip
-0.1ex\raise 0.5ex\hbox{\tiny ++}}}}

\newcommand{\ANDOR}[0]{And/Or}

\begin{document}
\pagestyle{plain}  

\title{Decomposition During Search for Propagation-Based Constraint Solvers}

\author{Martin Mann\inst{1}\fnmsep\thanks{Supported by the EU project EMBIO (EC contract
          number 012835)}
          \and Guido Tack\inst{2}
         \and Sebastian Will\inst{1}\fnmsep\thanks{Supported by the EU Network of Excellence
          REWERSE (project number 506779)}
         }
\institute{Bioinformatics, Albert-Ludwigs-University Freiburg, Germany \\
        \email{\{mmann,will\}@informatik.uni-freiburg.de}
        \and Programming Systems Lab, Saarland University, Germany\\
        \email{tack@ps.uni-sb.de}
        }

\maketitle

\begin{abstract} 
  We describe \emph{decomposition during search} (DDS), an integration of
  \ANDOR{} tree search into propagation-based constraint solvers. The
  presented search algorithm dynamically decomposes sub-problems of a
  constraint satisfaction problem into independent partial problems, avoiding
  redundant work.

  The paper discusses how DDS interacts with key features that make 
  propagation-based solvers successful: constraint propagation, especially for 
  global constraints, and dynamic search heuristics.

  We have implemented DDS for the Gecode constraint programming library. Two
  applications, solution counting in graph coloring and protein structure
  prediction, exemplify the benefits of DDS in practice.
  
   \end{abstract}

\section{Introduction}
\label{sec:intro}

Propagation-based constraint solvers owe much of their success to a
\emph{best-of-several-worlds} approach: They combine classic AI search methods
with advanced implementation techniques from the Programming Languages
community and efficient algorithms from Operations Research. Furthermore, the
CP community has developed a great number of propagation algorithms for global
constraints.

In this paper, we present how to integrate \ANDOR{} search into
propagation-based constraint solvers. We call the integration \emph{decomposition
during search} (DDS). We take full advantage of all the features mentioned above
that make propagation-based constraint solvers successful. The most interesting
points, and main contributions of our paper, are how DDS interacts with and
benefits from constraint propagation, especially in the presence of global
constraints, and dynamic search heuristics. We exemplify the profit of DDS by
exhaustive solution counting, an important application area of decomposing
search
strategies~\cite{Bayardo:_count_model_connec_compon:2000,Dechter:Mateescu:impact:CP2004}.

\pseudosubsubsec{Related work}
Only recently, counting and exhaustive enumeration of solutions of a constraint
satisfaction problem (CSP) gained a lot of
interest~\cite{Angelsmark:Jonsson:_improv_algor_count_solut:CP2003,Bayardo:_count_model_connec_compon:2000,Dechter:Mateescu:impact:CP2004,Roth:_hardness_approx_reasoning:AI1996}.
In general, the counting of CSP solutions is in the complexity class \#P, i.e.\ it
is even harder than deciding
satisfiability~\cite{Pesant:_count_solut_csps:IJCAI2005}. This class was defined
by Valiant~\cite{Valiant:TCS1979} as the class of counting problems that can be
computed in nondeterministic polynomial time. Notwithstanding the complexity,
there is demand for solution counting in real applications. For instance, in
bioinformatics counting optimal protein structures is of high importance for the
study of protein energy landscapes, kinetics, and protein
evolution~\cite{fitnesslandscapes,Wolfinger:etal:_explor:EPL2006} and can be
done using CP~{\cite{Backofen_Will_Constraints2006}}.

Already folklore, standard solving me\-thods for CSPs like Depth-First Search
(DFS) leave room for saving redundant work, in particular when counting
\emph{all} solutions~\cite{Freuder:Quinn:IJCAI1985}. Recent work by Dechter et
al.~\cite{Dechter:Mateescu:impact:CP2004,Dechter:IJCAI:2005} introduced
\ANDOR{} search for solution counting and optimization, which makes use
of repeated \emph{and}-decomposi\-tion during the search following a
pseudo-tree. Their work thoroughly studies and develops a rich theory of
\ANDOR{} trees.

While not in the context of general constraint propagation, similar ideas were
discussed before for
SAT-solving~\cite{Bayardo:_count_model_connec_compon:2000,Biere:Sinz:JSBMC2006,Li:Beek:SAT-decomp:2004}.
The SAT approaches also introduce the idea of analyzing the induced dependency
structure dynamically during the search. This avoids redundancy that occurs due
to the emergence of independent connected components in the dependency graph
during the search.

\pseudosubsubsec{Motivation and contribution}
The motivation for this paper is to tackle the same kind of redundancy for
solving very hard real world problems, such as the counting of protein
structures, that require a full-fledged constraint programming system. This
requests for a method which is tailored for integration into modern CP systems
and directly supports features such as global constraints and dynamic search
heuristics. To make use of the statically unpredictable effects of constraint
propagation and entailment, the presented method avoids redundant search
dynamically.

Our main contribution is to present how to integrate \ANDOR{} tree search into
a state-of-the-art, propagation-based constraint solver. This is exemplified
by extending the Gecode system~\cite{GecodeWebsite}. We describe
\emph{decomposition during search} (DDS) on different levels of abstraction,
down to concrete implementation details.

In detail, we stress the impact that constraint propagation has on
decomposability of the constraint graph, and how DDS interacts (and works
seamlessly together) with propagators for global constraints, the workhorses
of modern propagation-based solvers. We show that global constraint
decomposition is the key to enable the application of DDS, and discuss
techniques that enable global constraint decomposition. The practical value of
DDS in the presence of global constraints is shown empirically, using a well
integrated and competitive implementation for the Gecode library.

\pseudosubsubsec{Overview}
The paper starts with a presentation of the notations and concepts that are
used throughout the later sections. In Sec.~\ref{sec:dds}, we briefly
recapitulate \ANDOR{} search, and then present, on a high level of
abstraction, \emph{decomposition during search} (DDS), our integration of
\ANDOR{} search into a propagation-based constraint solver.
Sec.~\ref{sec:interaction} deals with the interaction of DDS with propagation
and search heuristics. Section~\ref{sec:global} discusses how global
constraints interact with DDS, focusing on decomposition strategies for some
important representatives.

On a lower level of abstraction, Sec.~\ref{sec:impl} sketches the concrete
implementation of DDS using the Gecode \CPP\ constraint programming library.
With the help of our Gecode implementation, we study the practical impact of
DDS in Sec.~\ref{sec:res} by counting solutions for random instances of two
important CSPs, graph coloring and protein structure prediction. Both examples
are hard counting problems (in class \#P). The study shows high average
speedups and reductions in search tree size, even using our prototype
implementation. These two sections therefore provide evidence that DDS can be
integrated into a modern constraint programming system in a straightforward 
and efficient way.
The paper finishes with a summary and an outlook on future work in 
Sec.~\ref{sec:disc}.

\section{Preliminaries}
\label{sec:defs}

This section defines the central notions that we want to use to talk about 
constraint satisfaction problems.

A \emph{Constraint Satisfaction Problem (CSP)} is a triple $P = (\X,\D,\C)$,
where $\X=\{x_1,\ldots,x_m\}$ is a finite set of variables, $\D$ a function of
variables to their associated value domains, and $\C$ a set of constraints. An
$n$-ary \emph{constraint $c\in \C$} is defined by the tuple of its $n$
variables $\vars(c)$ and a set of $n$-tuples of the allowed value
combinations.
We feel free to interpret $\vars(c)$ as the set of variables of $c$. A domain
$\D$ \emph{entails} a constraint $c$ if and only if all possible value combinations of
the variable domains of $\vars(c)$ in $\D$ are allowed tuples for $c$. A
\emph{solution of a CSP} is an assignment of one value $v\in \D(x_i)$ to each
variable $x_i\in\X$ such that all $c\in\C$ are entailed. The \emph{set of
solutions of a CSP $P$} is denoted by $\sol(P)$.

Based on these definitions some important properties of a CSP can be defined.
A \emph{CSP $P=(\X,\D,\C)$ is solved} if and only if $\forall x\in\X: \sizeof{\D(x)}=1$
and $\sol(P) \neq \emptyset$. \emph{P is failed} if and only if $\sol(P)=\emptyset$ and
\emph{satisfiable} otherwise. A CSP $P'$ is \emph{stronger than} $P$
($P'\sqsubseteq P$) if and only if $\sol(P')=\sol(P)$ and $\forall x\in\X':
\D'(x)\subseteq\D(x)$.

The \emph{constraint graph of a CSP $P=(\X,\D,\C)$} is a hypergraph 
$G_P=(V,E)$, where $V=\X$ and $E=\{ \vars(c) \mid c\in\C \}$.

\section{Decomposition During Search}
\label{sec:dds}

In this section, we recapitulate \ANDOR{} tree search.
Then, we present a high-level model of how to integrate 
\ANDOR{} tree search into a propagation-based constraint solver. We call 
this integration \emph{decomposition during search} (DDS).

\subsection{\ANDOR{} tree search}
Let us look at an example to get an intuition for \ANDOR{} search. Assume
$P=(\X,\D,\C)$ with $\X=\{ A,B,C,D \}$, $\D(A)=\{3,5\}$, $\D(B)=\{3,4\}$,
$\D(C)=\D(D)=\{1,2\}$, and $\C =$ `$A,B,C,D$ are pairwise different'.
Figure~\ref{fig:ST:CDFS} presents a corresponding search tree for a plain
depth-first tree search. Each node is a propagated sub-problem of $P$ and is
visualized as a constraint graph. As usual, a node is equivalent to the
\emph{disjunction} of all its children.

\begin{figure*}[t]
  \begin{center}
	  \includegraphics[width=0.7\textwidth]{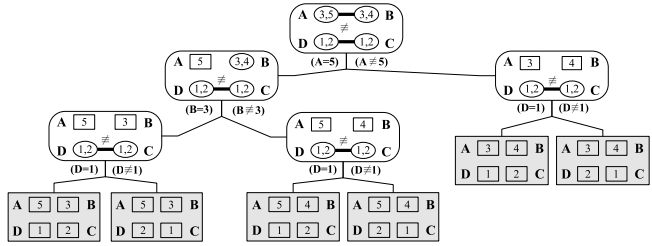}
	  \caption{DFS search tree.}
	  \label{fig:ST:CDFS}
  \end{center}
\end{figure*}

Even this tiny example demonstrates that plain DFS may perform redundant work:
The partial problem on the variables $C$ and $D$ is solved redundantly for
each solution of the partial problem on $A$ and $B$. We say that $\{A,B\}$ and
$\{C,D\}$ are \emph{independent} sets of variables.

The central idea of \ANDOR{} tree search~\cite{Dechter:Mateescu:impact:CP2004} is
to detect independent partial problems \emph{during search}, to enumerate partial
solutions of the partial problems independently, and finally to combine them to
solutions, or to compute the number of solutions. That way, each independent
partial problem is searched only once. The name \ANDOR{} search stresses that the
search tree contains both disjunctive, choice nodes (OR) and \emph{conjunctive
nodes} (AND), representing decompositions. Figure~\ref{fig:ST:CDDS} shows a
search tree for the same CSP as in Figure~\ref{fig:ST:CDFS}, but using \ANDOR{}
search. For now, you can read the big $X$ as ``combine''. Here, the search tree
contains only one decomposition and two choices, instead of five choices in
Figure~\ref{fig:ST:CDFS}. In general, the amount of redundant work can be
exponential in the size of the CSP.

\begin{figure*}[t]
\begin{center}
    \includegraphics[width=0.6\textwidth]{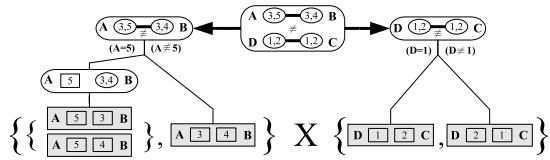}
    \caption{\emph{And/or} search tree.}
    \label{fig:ST:CDDS}
\end{center}
\end{figure*}

\subsection{Integrating \ANDOR{} search into a pro\-pa\-ga\-tion-based solver}
\label{sec:algo}

From a bird's eye view, \ANDOR{} search can be done easily in the context
of propagation-based constraint solving. Algorithm~\ref{algo:CDDS} counts all
solutions of a CSP~$P$, decomposing the problem where possible.

Ignoring line~5 for a moment, the algorithm runs a standard depth-first search
(DFS). The function $\textsc{Propagate}$ in line~2 performs constraint
propagation: it maps a CSP $P=(\X,\D,\C)$ to a CSP $P'=(\X,\D',\C')$ such that
$P'\sqsubseteq P$. $\textsc{Propagate}$ may remove entailed constraints from
$\C$.
If $P'$ is failed or solved, we just return that we found no resp.\ one solution
(lines~3,4). Otherwise (line~6), we split the problem into
\textsc{LeftChoice}($P'$) and \textsc{RightChoice}($P'$). These functions 
implement the search heuristic and will be discussed in more detail in
Sec.~\ref{sec:search}. As the branches correspond to a \emph{disjunction} of
$P'$, the recursive counts \emph{add} up to the total number of solutions.

\begin{algorithm}[h]
\caption{Counting by Decomposition During Search}
\label{algo:CDDS}
\begin{scriptsize}
\begin{algorithmic}[1]
  \Function{DDS}{$P$}
  \State $P' \gets$ \Call{Propagate}{$P$}
  \State \textbf{if} {\Call{IsFailed}{$P'$}} \Return 0
  \State \textbf{if} {\Call{IsSolved}{$P'$}} \Return 1 \
  \State \textbf{if} {\Call{Decompose}{$P'$}$=(P'_1,P'_2)$} \
    \Return\
    \Call{DDS}{$P'_1$}\ $\cdot$\ 
    \Call{DDS}{$P'_2$}
  \Statex \Comment{decomposition}
  \State \Return\
    \Call{DDS}{\textsc{LeftChoice}($P'$)}\ \ +\ \
    \Call{DDS}{\textsc{RightChoice}($P'$)}
  \Statex \Comment{choice}
  \EndFunction
\end{algorithmic}
\end{scriptsize}
\end{algorithm}

The only addition that is necessary to turn DFS into an \ANDOR{} search is line~5:
If the problem can be decomposed into $P'_1$ and $P'_2$ (we simplify by assuming
only binary decomposition), these partial problems in \emph{conjunction} are
equivalent to $P'$. Hence, we \emph{multiply} the results of the recursive calls.

In contrast to investigating decomposability only on the initial CSP for a
static variable selection, Algorithm~\ref{algo:CDDS} follows a dynamic
approach: the check for decomposability is interleaved with propagation and
normal search. As search progresses, more and more variables are assigned to
values, and more and more constraints are detected to be entailed. We shall
see that this greatly increases the potential for decomposition in
Sec.~\ref{sec:interaction}. Furthermore, decomposition is completely
independent of the implementation of $\Call{LeftChoice}{}$ and
$\Call{RightChoice}{}$, so any search heuristic can be used.

\pseudosubsubsec{Short-circuit} As a straightforward optimization, we can employ
short-circuit reasoning in line~5. If $\Call{DDS}{P'_1}$ returns no solutions, we
do not have to compute $\Call{DDS}{P'_2}$ at all. Note the potential pitfall
here: There are situations where DFS detects failure easily, but DDS has to
search a huge partial problem $P'_1$ before detecting failure in $P'_2$. We come
back to this in Sec.~\ref{sec:search}.

\pseudosubsubsec{Enumerating solutions} Extending DDS to enumeration of solutions
is straightforward: We just have to return an empty list in case of failure
(line~3), a singleton list with a solution when we find one (line~4), and
interpret addition as list concatenation and multiplication as combination of
partial solutions. Instead of enumeration, we can also build up a tree-shaped
compact representation of the solution space (as in Fig.~\ref{fig:ST:CDDS} and
later in Fig.~\ref{fig:hpenergie}c).

In the rest of this section, we show how to compute the $\Call{Decompose}{}$
function efficiently.

\subsection{Computing the decomposition}

We will now define formally when a CSP can be decomposed, and give a
sufficient algorithmic characterization that leads to an efficient
implementation.

\pseudosubsubsec{Restriction and independence}
The \emph{restriction of a function $f:\Y\rightarrow \Z$ to a set
$\X\subseteq\Y$} is defined as 
\begin{equation*}
	\restrict{f}{\X}: \X\rightarrow \Z, x \mapsto f(x).
\end{equation*}

We define the restriction of a CSP $P=(\X,\D,\C)$ to a set of variables
$\X'\subseteq \X$ by $\restrict{P}{\X'} =
(\X',\restrict{\D}{\X'},\restrict{\C}{\X'})$, where 
\begin{equation*}
	\restrict{\C}{\X'} = \{ c\in\C \mid \vars(c)\subseteq\X' \}.
\end{equation*}

A non-empty proper subset $\hat\X\subset\X$ is \emph{independent} in a CSP
$P=(\X,\D,\C)$, if and only if 
\begin{equation*}
	\sol(P) = \emptyset \text{ or }\restrict{\sol(P)}{\hat\X} =
	\sol(\restrict{P}{\hat\X}).
\end{equation*}
For $\hat\X$ independent in $P$, we say that
$\restrict{P}{\hat\X}$ is a \emph{partial problem of $P$}. We can decompose $P$
if it has a partial problem.

The key to an efficient implementation of $\Call{Decompose}{}$ is to have an
algorithmic interpretation of what it means that a CSP can be decomposed into
partial problems. We now show that connected components in the constraint
graph of a CSP represent independent partial problems.

A graph is \emph{connected} if and only if there exists a path between all nodes. A
\emph{connected component} of a constraint graph $G_P$ is a maximal connected
subgraph.

\newtheorem{prop}{Proposition}
\begin{prop}
	Consider a CSP~$P = (\X,\D,\C)$ with constraint graph $G_P$. If $\XH\subset\X$
	is a connected component in $G_P$, then $\restrict{P}{\XH}$ is a partial problem of $P$.
\end{prop}

\begin{proof}
There exists no hyperedge between
node $x\in\XH$ and node $y\notin\XH$, as connected components
are maximal. This means that there is no constraint between any $x$ and $y$
in $P$. We have to distinguish two cases: If $P$ is unsatisfiable, $\XH$ is
trivially independent (by definition of independence). Otherwise, take an
arbitrary solution $\hat s\in\sol(\restrict{P}{\XH})$, and an arbitrary
solution $s\in\sol(P)$. Merging $\hat s$ into $s$ yields $s' = (x\mapsto
\hat s(x)$ for $x\in\XH$, $x\mapsto s(x)$ otherwise). This $s'$ is again a
solution of $P$, as all constraints on $\X\setminus\XH$ are still satisfied,
and all constraints on $\XH$ are satisfied, too. As we picked $s$ and $\hat s$
arbitrarily, we get $\restrict{\sol(P)}{\XH} \supseteq
\sol(\restrict{P}{\XH})$. Because $\restrict{C}{\XH}$ of $\restrict{P}{\XH}$
covers all constraints of $\C$ restricting $\XH$, it follows
$\restrict{\sol(P)}{\XH} \subseteq \sol(\restrict{P}{\XH})$.\\ 
\noindent Therefore, it holds $\restrict{\sol(P)}{\XH} =
\sol(\restrict{P}{\XH})$.
\end{proof}

This result is not
new~\cite{Bayardo:_count_model_connec_compon:2000,Freuder:Quinn:IJCAI1985},
but we repeat it to illustrate the central algorithmic idea. Connected
components can be computed in linear time in the size of the graph, and
incremental algorithms are available. We can thus implement
$\Call{Decompose}{}$ as a simple connectedness algorithm on the constraint
graph that yields all partial problems of the current CSP.

Finding more than one non-empty connected component is a sufficient condition for
finding partial problems, but not a necessary one. As an example, consider the
CSP that contains the trivial constraint allowing all combinations of values for
$x$ and $y$. Then $x$ and $y$ may still be independent, but the constraint graph
shows a hyperedge connecting the two variables, so that $x$ and $y$ will always
end up in the same connected component. In the following section, we will see how
propagation-based solvers can deal with this.

\section{How DDS Interacts With Propagation and Search}
\label{sec:interaction}

The previous section showed how DDS can be integrated into a propagation-based
solver. But what are the consequences, how is decomposition affected by
propagation and search, and how can it benefit from the search heuristic?

\subsection{Constraint graph dynamics}

Decomposition examines the constraint graph \emph{during search}. This is
vital as propagation and search \emph{modify} the constraint graph
dynamically -- they narrow the domains of the problem's variables and remove
some entailed constraints. The result is a sparser constraint graph with more
potential for decomposition:

\pseudosubsubsec{Assignment} Clearly, an assigned variable ($|\D(x)|=1$) is
independent of all other variables of the CSP. This implies that connections of
hyperedges to assigned variables can be removed from the constraint graph --
the constraint graph becomes sparser. Assignment increases the potential for
decomposition, since an assigned variable may have been responsible for keeping
two otherwise independent parts of the graph connected.

\pseudosubsubsec{Entailment} 
Consider the example CSP $P=(\X,\D,\C)$ from the end of the previous section, where
variables $x\in\X$ and $y\in\X$ are connected by a trivial constraint~$c\in\C$
allowing all possible value tuples. Obviously, $c$~is entailed in~$P$, it will
not contribute to propagation any more. Formally, we have for
$P'=(\X,\D,\C\setminus\{c\})$ that $\sol(P')=\sol(P)$. It is also obvious that
the constraint graph for $P'$ is sparser than the one for $P$, it contains one
edge less. It is thus vital to our approach to detect entailment of constraints
as early as possible, and to remove entailed constraints from~$\C$. Clearly, full
entailment detection is coNP-complete. Most CP systems (e.g.\ Gecode) however
implement a weak form of entailment detection in order to remove propagators
early, which our approach automatically benefits from.

\subsection{Search heuristics}
\label{sec:search}

The applied search heuristic, encoded by \Call{LeftChoice}{} and
\Call{RightChoice}{}, is extremely important for the efficiency of the search.
In particular, dynamic heuristics, natively supported by DDS, are known to be
largely superior to static ones.

In the following (and for our implementation) we refer to the common
\emph{variable-value heuristics} that select a variable~$x$ and a value
$v\in\D(x)$. Afterwards, the disjunction is done by $\Call{LeftChoice}{P} =
(\X,\D,\C\cup(x\diamond v))$ and $\Call{RightChoice}{P} =
(\X,\D,\C\cup\lnot(x\diamond v))$, where $\diamond$ is some binary relation. 

\pseudosubsubsec{Variable selection}
The variable selection strategy of such a heuristic is crucial for the size of
the search tree and often problem specific. Nevertheless, a common method for
variable selection is `first-fail', which selects by minimal domain size. Other
strategies use the degree in the constraint graph or the minimal/maximal/median
value in the domain. Static variable orderings are in many cases inferior to
these dynamic strategies. To gain the best search performance by DDS, the
variable selection further has to induce constraint graph decompositions as early
as possible.

\pseudosubsubsec{Heuristics that maximize decomposition} In order to maximize the
number of decompositions during search, the heuristic can be guided by the
constraint graph. Such a heuristic may compute e.g.\ cut-points, bridges, or more
powerful (minimal) cut-sets/separators~\cite{MarinescuD06_ECAI06}. Our framework
is well prepared to accommodate such complex strategies, in particular because
our method already builds on access to the constraint graph. 

An open problem is the possible contradiction of heuristics aimed at decomposition versus fine-tuned problem specific heuristics. The heuristic aimed at decomposition may
yield many partial problems that are not satisfiable and therefore lead to an
inefficient search. On the other hand, the problem specific heuristic might yield
no or only a few decompositions, which makes DDS unprofitable. Therefore, no
general rule can be given. But our experiments suggest that a hybrid heuristic of
selecting the variable with the highest node degree, and using the problem-specific
heuristic as a tie breaker, yields good results (see Sec.~\ref{sec:res}).

\pseudosubsubsec{Order of exploration of partial problems} 
It is of high importance in which order the children of \emph{and}-nodes, the partial problems, are explored. After detecting inconsistency in one partial problem, the exploration
of the remaining partial problems is needless. Given an unsatisfiable problem, a
good variable selection heuristic yields a failure during search as soon as
possible to avoid unnecessary search (this is called the \emph{fail-first}
principle). Assuming that we already have such a good variable selection
heuristic, we use it to guide the partial problem ordering of DDS: (1) apply the
selection to the whole (non-decomposed) variable set and (2) choose the partial
problem first that contains the selected variable. That way, a good heuristic
will lead to failure in the first explored partial problem if the decomposed
problem has no solution. Further, it decreases the probability that the remaining
partial problems are not satisfiable (in case the first is). This mitigates the
effect that failure may be detected late using DDS, mentioned in
Subsec.~\ref{sec:algo}.

\section{Global Constraints}
\label{sec:global}

One of the key features of modern constraint solvers is the use of global
constraints to strengthen propagation. Therefore, a search algorithm has to
support global constraints in order to be practically useful in such systems.
We describe the problems global constraints pose for DDS, and how to tackle
them.

For an $n$-ary (global) constraint, the constraint graph contains a hyperedge
that connects all $n$ variable nodes. Consider a CSP for
$\textsf{All-different}(w,x,y,z)$ with $w,x\in\{0,1\}$, $y,z\in\{2,3\}$. In this
model, we cannot decompose, although the binary constraints $a\neq b$ for
$a\in\{w,x\}$, $b\in\{y,z\}$ are entailed! Thus, in the current set-up, the
global \textsf{All-different} unnecessarily prevents decomposition.

The solution to this problem is to take the \emph{internal} structure of
global constraints into account: the global constraint itself can be
decomposed into smaller constraints (on fewer variables). We will call this
the \emph{constraint decomposition} to distinguish it from \emph{constraint
graph decomposition}.

If we reflect the constraint decomposition of global constraints in the
constraint graph, we recover all the decompositions that were prevented by the
global constraints before.

For our example involving the \textsf{All-different} constraint, the
constraint graph would contain two constraints, $w\neq x$ and $y\neq z$,
instead of the one global \textsf{All-different}. The graph can now be decomposed.

As global constraints often cover a significant portion of the variables in a
problem, global constraint decomposition is an essential prerequisite to make a
constraint graph decomposition by DDS possible. Note that global constraint
decomposition is independent of the other constraints present in the constraint
graph; typical permutation problems, for instance, feature one global
\textsf{All-different} that forces all variables to form a permutation, and then
several other constraints that determine the concrete properties of the
permutation. Applying DDS to such a problem is useless unless the constraint
decomposition of the \textsf{All-different} is considered when computing
connected components.

In general, a global constraint can be decomposed if and only if its
\emph{extension} (the set of allowed tuples) can be represented as a non-trivial
product, i.e.\ a product of non-singleton sets. For the above example, the set of
allowed tuples is $\{(0,1),(1,0)\}\times\{(2,3),(3,2)\}$. If a constraint can be
represented as such a non-trivial product $a\times b$, we can decompose it into
two independent constraints, one with the tuples of $a$, the other with the
tuples of $b$.

Formally, defined in terms of CSP solutions, we say a constraint $c\in\C$ of a
CSP $P=(\X,\D,\C)$ is decomposable into the constraints $c_1,\ldots,c_k$, if and only if
\begin{eqnarray*}
	vars(c) & = & \bigcup_{i\in[1,k]} vars(c_i) \quad\text{and} \\
	\forall_{i\not= j \in [1,k]} & : & vars(c_i) \cap vars(c_j) = \emptyset 
	 \quad\text{and} \\
	sol(P) & = & sol(P')
\end{eqnarray*}
with $P'=(\X,\D,\C')$ and $\C'=\{\;c_1,\ldots,c_k\}\cup\C\setminus\{\;c\}$.

\subsection{Non-decomposable constraints}
Some global constraints are never decomposable during a constraint search, since
they cannot be decomposed for any arbitrary domain $\D$.

For example, the tuples that satisfy a linear constraint, such as a linear
equation or inequation, can never be represented as a non-trivial product. The
reason for this is that each variable in a linear constraint \emph{functionally
depends} on all other variables: for any two variables $x_i$ and $x_j$ in the
constraint $\sum_{i=1}^n x_i=c$, picking a value for $x_i$ determines $x_j$,
when all other variables are assigned. Therefore, we cannot arbitrarily pick values
from the domains $\D(x_i)$ and $\D(x_j)$ such that the constraint is satisfied.

Consequently, subgraphs covered by linear constraints stay connected until all
its variables are assigned. When solving a problem where subsets of $\X$ are
covered by non-decomposable constraints, we can guide the search heuristic
for a decomposition into these subsets. On the other hand, in problems where
one non-decomposable constraint covers the whole variable set, it is obvious
that search cannot profit from DDS.

\subsection{Decomposable constraints}

To take full advantage of DDS, efficient algorithms for the detection of possible
constraint decompositions are necessary. Good candidates for an integration are
the propagation methods that already investigate the variable domains.
Unfortunately, these propagators are highly constraint specific and thus no single
general detection procedure for all decomposable constraints can be given. But
there are global constraints whose propagation algorithms and data structures
either directly yield the decomposition as a by-product, or that can be easily
extended for detection.

In the following, we discuss three important propagators that can detect and
compute a decomposition efficiently.

\pseudosubsubsec{All-different} In contrast to linear constraints, there is no
functional dependency between variables for \textsf{All-different}, but the exact
inverse -- variables depend on the \emph{absence} of values in a domain. This
yields a maximal potential for decomposition. R\'{e}gin's propagator for
\textsf{All-different}~\cite{Regin:94} employs a \emph{variable-value} graph,
connecting each variable node with the value nodes corresponding to the current
domain. We can observe that \textsf{All-different} can be decomposed if and only
if the variable-value graph contains more than one connected component. In
Fig.~\ref{fig:alldiff} the variable-value graph for the simple initial example is
shown. As connected components of this graph have to be computed anyway during
the propagation, we can get the constraint decomposition without any additional
overhead. This technique generalizes to the \textbf{global cardinality
constraint}.

\begin{figure}[htb]
\begin{center}
	\includegraphics[width=8em]{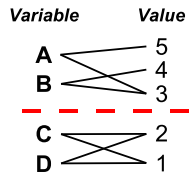}
\end{center}
\caption{Decomposable variable-value graph of an \textsf{All-different}
constraint.}
\label{fig:alldiff}
\end{figure}

\pseudosubsubsec{Slide} Introduced by Bessiere et al.~\cite{BessiereSlide},
\textsf{Slide} slides a $k$-ary constraint $c$ over a sequence of variables,
i.e.\ it holds if $c(x_i,\dots,x_{i+k-1})$ holds for all $1\leq i\leq n-k+1$.
\textsf{Slide} can be split into two at variable $x_d$ if the individual
constraints involving $x_d$ are entailed (see Fig.~\ref{fig:slide}). Entailment
happens at the latest when all variables between $x_{d-k+1}$ and $x_{d+k-1}$ are
assigned. Depending on how soon the individual $c$ are entailed, we can decompose
even earlier. \textsf{Slide} establishes a dependency of a \emph{fixed width}
between variables, so that once that width is reached, the constraint can be
decomposed. Note, this constraint decomposition is not complete and only detects
certain non-trivial products along the variable ordering. Computing the full
constraint decomposition would require insight into the structure of the
constraint $c$.

\begin{figure}[htb]
\begin{center}
    \includegraphics[width=10em]{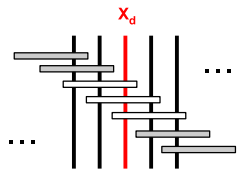}
\end{center}
\caption{Decomposing \textsf{Slide} at variable $x_d$.}
\label{fig:slide}
\end{figure}

\pseudosubsubsec{Regular}
Pesant's regular language membership constraint \cite{PesantRegular} states
that the values taken by a sequence of variables $x_1,\dots,x_n$ belong to a regular
language~$L$. It is essentially a constraint represented in extension, as
arbitrary tuples can be encoded into regular expressions. Propagation works on an
unfolding of a finite automaton accepting $L$,
called the \emph{layered graph} (see Fig.~\ref{fig:regular}).

If now, at some point during propagation, one layer is left with a single state
(see Fig.~\ref{fig:regular}), the graph can be split into two halves, making the
singleton state a new final state (for the left half) and start state (for the
right half). They correspond to regular expressions $R_l$ and $R_r$, covering the
two substrings left and right of that layer, such that the language generated by
$R_lR_r$ is a sublanguage of $L$ that contains exactly those strings still
licensed by the variable domains. Note that constraint decomposition is possible
even without the variables $x_i$ being assigned, but that it heavily depends on
the actual automaton. Again, as for \textsf{Slide}, this only detects those
non-trivial products that are compatible with the variable ordering of the
regular constraint. Determining the full decomposition for \textsf{Regular} would
amount to finding a non-trivial product representation of its allowed tuples,
which cannot be computed efficiently.

\begin{figure}[htb]
\begin{center}
   \includegraphics[width=12em]{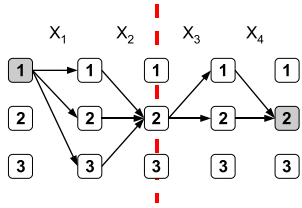}
\end{center}
\caption{Layered graph for a \textsf{Regular} constraint.}
\label{fig:regular}
\end{figure}

\section{Implementation}
\label{sec:impl}
Our implementation of DDS extends Gecode, a \CPP\ constraint programming library.
In this section, we give an overview of relevant technical details of Gecode, and
discuss the four main additions to Gecode that enable DDS: access to the
constraint graph, decomposing global constraints, integrating
$\Call{Decompose}{}$ into the search heuristic, and specialized search engines.
The additions to Gecode comprise only 2500~lines (5\%) of \CPP code and enable
the use of DDS in any CSP modeled in Gecode. DDS will be available as part of the
next release of Gecode.

\subsection{Gecode}
The Gecode library~\cite{GecodeWebsite} is an open source constraint solver 
implemented in \CPP. It lends itself to a prototype implementation of DDS 
because of four facts:
\begin{enumerate}
  \item Full source code enables changes to the available 
propagators.
  \item The reflection capabilities allow access to the constraint 
graph.
  \item Search is based on recomputation and copying, which significantly 
eases the implementation of specialized branchings and search engines.
  \item It 
provides good performance, so that benchmarks give meaningful results.
\end{enumerate}

\subsection{Constraint graph}
In most CP systems, the constraint graph is implicit in the data structures
for variables and propagators. Gecode, e.g., maintains a list of propagators,
and each propagator has access to the variables it depends on.

For DDS, a more explicit representation is needed that supports the
computation of connected components. We can thus either maintain an
additional, explicit constraint graph during propagation and search, or
extract the graph from the implicit information each time we need it. For the
prototype implementation, we chose the latter approach. We make use of
Gecode's reflection API, which allows to iterate over all propagators and
their variables. Through reflection, we construct a graph using data
structures from the boost graph library~\cite{boostlib}, which also provides
the algorithm that computes connected components.

Assigned variables are independent of all other variables as discussed in
Sec.~\ref{sec:interaction}. Therefore, they are reported as individual partial
problems (connected components) but are ignored to avoid useless trivial
decompositions without any effect. Instead these already solved single variable
CSPs are added to an arbitrary partial problem that covers at least one
unassigned variable. If the final number of such ``non-solved'' partial problems
is at least two, a problem decomposition is initialised. This significantly
speeds up the search process because only profitable decompositions are done.

\subsection{Global constraint decomposition}
As discussed in Sec.~\ref{sec:global}, it is absolutely essential for the success
of DDS to consider constraint decompositions of global constraints when computing
the connected components.

There are two possible implementation strategies for decomposing global
constraints. A propagator can either detect decomposability during propagation
and replace itself with several propagators on subsets of the variables. Or,
alternatively, the constraint decompositions are only computed on demand when the
constraint graph is required for connected component analysis. We implemented the
latter option.

This again leaves two possible implementations. When the constraint graph is
decomposed, one propagator (for a global constraint) may belong to two connected
components. When search continues in the individual components, we can either use
the propagator as it is in both components, or replace it by its decomposition.
The latter option has the advantage that the smaller propagator may be more
efficient (as it can ignore the variables outside its connected component).
However, for simplicity, we implemented the former.

\subsection{Decomposing branchings}
Once we have identified connected components in the constraint graph, we have
to create the partial problems that correspond to these components. In Gecode,
we exploit the duality of choice and decomposition: both add branches to the
search tree. The following observation leads to a simple and efficient
implementation. If the heuristic is restricted such that it only selects
variables inside one connected component, also propagation will only occur for
variables of that component: For $\hat\X$ independent in $P$,
$\restrict{\propagate(P)}{\hat\X}=\propagate(\restrict{P}{\hat\X})$.

For our Gecode implementation, $\Call{Decompose}{}$ is thus realized as a
\emph{branching}. A branching in Gecode usually implements
\Call{LeftChoice}{}/\Call{RightChoice}{}. For DDS, we extend it to also
implement \Call{Decompose}{}: If decomposition is possible, the branching
limits further search to the variables in one connected component per branch.
Otherwise, it just creates the usual choices according to the heuristic.

Branchings in Gecode are fully programmable. They have to support two
operations\footnote{In fact, branchings in Gecode have a slightly more
complex interface, which we deviate from to simplify presentation.}:
\Call{description}{} and \Call{commit}{}. \Call{description}{} returns an
abstract description of the possible branches while \Call{commit}{} executes
the branching according to a given description and alternative number.

A decomposing branching in Gecode is a wrapper around a standard
variable-value branching. The actual work is done by \Call{description}{}: it
requests the constraint graph and performs the connected component analysis.
If decomposition is possible, a special description is returned, representing
the independent subsets $\hat{\X_i}\subset \X$. Otherwise,
\Call{description}{} is delegated to the embedded variable-value branching.
Note that Gecode supports $n$-ary branchings, so decompositions do not 
have to be binary (as presented so far).

When \Call{commit}{} is invoked with a variable-value description, the call is
again delegated to the embedded branching. For a decomposition description,
the branching's list of variables is updated to $\hat{\X_i}$ for branch~$i$,
those still active in the selected component.

\subsection{Decomposition search engines}
As decomposition is performed by the branching, the search engines have to be
specialized accordingly.
We developed four search engines for DDS. A counting
search engine computes the number of solutions of a given problem. A
general-purpose search engine allows to incrementally search the whole tree
and access all the partial solutions. Based on that we provide a search engine
that enumerates all full solutions. A graphical search engine based on
Gecode's \emph{Gist} (graphical interactive search tool) displays the search
tree with special decomposition nodes, and allows to get an overview of where
and how a particular problem can be decomposed. Figure~\ref{fig:gist} shows a
screen shot of a partial search using DDS. Circular nodes with inner squares
represent decompositions. All search engines accept cut-off parameters for the
number of (full, not partial!) solutions to be explored.

\begin{figure}[hbt]
  \centering
  \includegraphics[width=4cm]{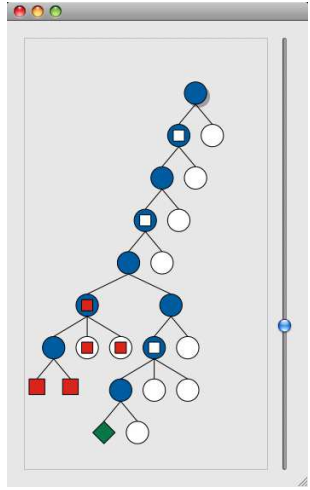}\caption{Gist, with decomposition nodes}\label{fig:gist}
\end{figure}

\section{Applications and Empirical Results}
\label{sec:res}
To illustrate possible use cases of DDS we applied it to two counting problems
with global constraints. At first, the widely known graph coloring problem
allows for a good and scalable illustration of the DDS effects due to the
coherence of problem and constraint graph structure. It serves therefore as a
model to investigate the impact of the problem structure on DDS in the
presence of global constraints. Afterwards, we study the benefit of DDS on
the real world problem of optimal protein structure prediction. This problem can
be modelled using constraint programming~\cite{Backofen_Will_Constraints2006}
but necessitates the presence of global constraints covering the whole problem.
Thus, the discussed constraint decomposition is an essential prerequisite to
enable DDS.

Both applications show tremendous reductions in runtime and search 
tree size.

The applications were realized using our DDS implementation in Gecode. Only
the search strategy was changed (DFS to DDS) -- modeling, variable and value
selection were kept the same for an appropriate comparison of the results. We
chose maximal degree with minimal domain size as tie breaker as dynamic
variable selection, which enforces decomposition and works well for DFS, too.

\subsection{Graph coloring}

Graph coloring is an important and hard problem with applications in scheduling,
assignment of radio frequencies, and computer
optimization~\cite{narayanan01static,GC_frequency,GC_scheduling}. A proper
coloring assigns different colors to adjacent nodes. We want the chromatic
polynomial for the chromatic number, i.e.\ the number of graph colorings with
minimal colors. Graph coloring is a useful benchmark, because it gives us a
scalable problem, so that we can apply DDS to instances of varying complexity.

\pseudosubsubsec{The constraint model}
For a given undirected graph $g$ and a number of colors $c$ we introduce one
variable per node with the initial domains $0..(c-1)$. For each maximal clique
of size $> 2$, we post an \textsf{All-different} constraint on the corresponding
variables. This maximizes the propagation necessary to solve these problems
but still enables DDS as we discuss below. For all remaining edges we add
binary inequality constraints.

\pseudosubsubsec{The test sets}
We generated the two test sets GC-30 and GC-50 of graphs with 30~and 50~nodes.
For each size, random graphs were obtained by inserting an edge of the
complete graph with a fixed uniform edge probability $P^e$. This was done
using the Erd\H{o}s-R\'enyi random graph generator GTgraph~\cite{gtGraph}. For
each edge probability $P^e$ from 16~to 40~percent, 2000~graphs were generated
and their colorings counted via DFS and DDS. To test highly degenerated
problems (with many solutions) as well, we stopped after 1~million solutions.

\pseudosubsubsec{Results}
For the test sets, Tab.~\ref{tab:res:gc} compares the time consumption and 
search tree size by average ratios of DFS and DDS ($\frac{DFS}{DDS}$). A 
figure of 100 thus means that DDS is 100 times faster than DFS, or that the 
DFS search tree has 100 times as many nodes as the one for DDS. A dash means 
that most of the problems were not solved within a given time-out.

\begin{table}[hbt]
\centering
\footnotesize
        \begin{tabular}{ll|c|c|c|c}
         {\scriptsize DFS/DDS}: & Test set & 16 \% & 18 \% & 20 \% & 22 \%  \\ 

         \hline
         \multirow{2}{1.2cm}{rel. RT:} & GC-30 & 411.2 & 197.7 & 75.74 &
         34.6 \\ 
         & GC-50 & 242.7 & 151.8 & 34.23 & 16.5  \\ 

         \hline
         \multirow{2}{1.2cm}{ST size:} & GC-30 & 680.3 & 344.4 & 142.0 & 74.48
          \\ 
         & GC-50 & 646.1 & 383.8 & 94.28 & 47.26 \\
         \\
         \\
         {\scriptsize DFS/DDS}: & Test set  & 24 \% & 28
         \% & 32 \% & 40 \% \\ 

         \hline
         \multirow{2}{1.2cm}{rel. RT:} & GC-30 & 23.1 & 11.9 & 3.85 &
         2.14 \\
         & GC-50   & 18.2 & 3.4 & 2.71  & --  \\

         \hline
         \multirow{2}{1.2cm}{ST size:} & GC-30  & 62.27 & 33.96 & 10.90  & 4.97  \\ 
         & GC-50 & 41.69  & 11.6 & 9.28 & -- \\
         \\
        \end{tabular}
\caption{Average ratios of DFS vs. DDS for various edge probabilities (RT =
runtime, ST = search tree)}
\label{tab:res:gc}
\end{table}

The presented runtime ratios show the high speedup for graphs with edge 
probabilities $P^e \leq 40$\%. The distribution of speedup is exemplified in 
Fig.~\ref{fig:gc-hist}. The speedup corresponds to an even larger reduction of 
the search tree for DDS, which was only increased for~0.5\% of all problems. 
Furthermore, sparse graphs yield a much higher runtime improvement than dense 
graphs, visualized by Fig.~\ref{fig:speedup}. The number of fails and 
propagations show no significant effect of DDS in contrast to runtime or search
tree size.

\begin{figure}[hbt]
\begin{center}
    \includegraphics[width=0.4\textwidth]{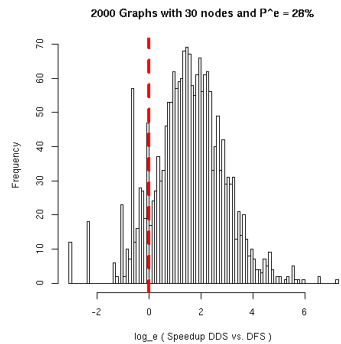}
\caption{
Histogram of logarithmic speedup for $P^e = 28$ and~30 nodes (the dashed line
marks equal runtime).}
\label{fig:gc-hist}
\end{center}
\end{figure}

\begin{figure}[hbt]
\begin{center}
    \includegraphics[width=0.4\textwidth]{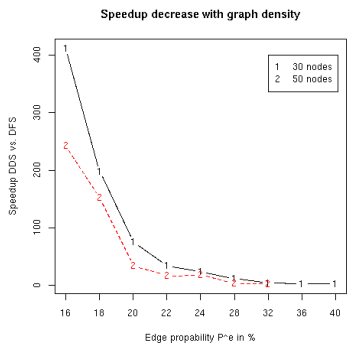}
\caption{Avg. speedup decrease from~400 to~2 by graph density.}
\label{fig:speedup}
\end{center}
\end{figure}

Still, the search tree reduction is not completely reflected in runtime 
speedup, which illustrates the computational overhead of DDS in the current 
prototypic implementation. Anyway, our data shows that DDS is well suited to 
improve solution counting even for dense graphs with $P^e$ about~40\%. We 
expect even higher speedups and search-tree reductions if the solutions are 
counted completely, i.e.\ without the current upper bound of 1~million.
Table~\ref{tab:res:gc} suggests that the speedup decreases with increasing 
number of nodes to color in the graph. With increasing number of nodes, the 
graph \emph{as well as the constraint graph} grow quadratically. 

The speedup is significantly lower than the reduction of the search tree. In
part this can be ascribed to our implementation that rebuilds the constraint
graph in each search step. A system that provides cheaper constraint graph
access, e.g.\ by maintaining it incrementally, is expected to perform much
better.

\subsection{Optimal protein structure prediction}

The prediction of optimal (minimal energy) structures of simplified lattice 
proteins is a hard (NP-complete) problem in bioinformatics. Here we focus on 
the HP-model introduced 
by Lau and Dill~\cite{Lau:_lattic_statis_mechan_model_confor:ACS1989}. In this
model, a protein chain is reduced to a sequence of monomers of equal size, whereby the 
20~aminoacids are grouped into hydrophobic~(H) or polar~(P). A structure is a 
self-avoiding walk of the underlying lattice (e.g.\ square or cubic). A contact 
energy function is used to determine the energy of a structure. The energy 
table and an example is given in Fig.~\ref{fig:hpenergie}.
The problem is to predict minimal energy structures for a given HP-sequence.

The number and quality of optimal structures has applications in the study of 
energy landscape properties, protein evolution and 
kinetics~\cite{flamm02:_barrier_trees_degen_lands,fitnesslandscapes,Wolfinger:etal:_explor:EPL2006}.

\begin{figure}[hbt]
\centering
\begin{tabular}{ll}
     a)
    \begin{minipage}[t]{16mm}
        \begin{tabular}[b]{c|cc}
            & H & P \\
            \hline
            H & -1 & 0\\
            P & 0 & 0\\
        \end{tabular}
    \end{minipage} 
	\hspace{1em}
	&
     b)
    \begin{minipage}[t]{25mm}
        \includegraphics[width=20mm]{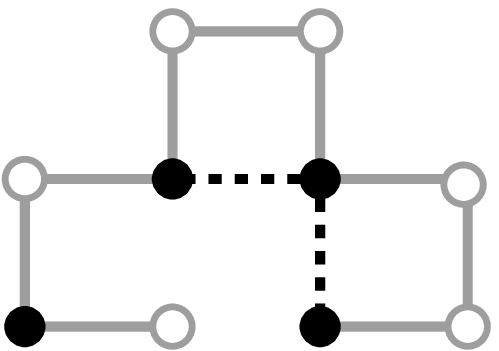}
    \end{minipage}
	\newline
	\\
	\\
	\multicolumn{2}{l}{
     c)
    \begin{minipage}[b]{45mm}        
         \includegraphics[width=45mm]{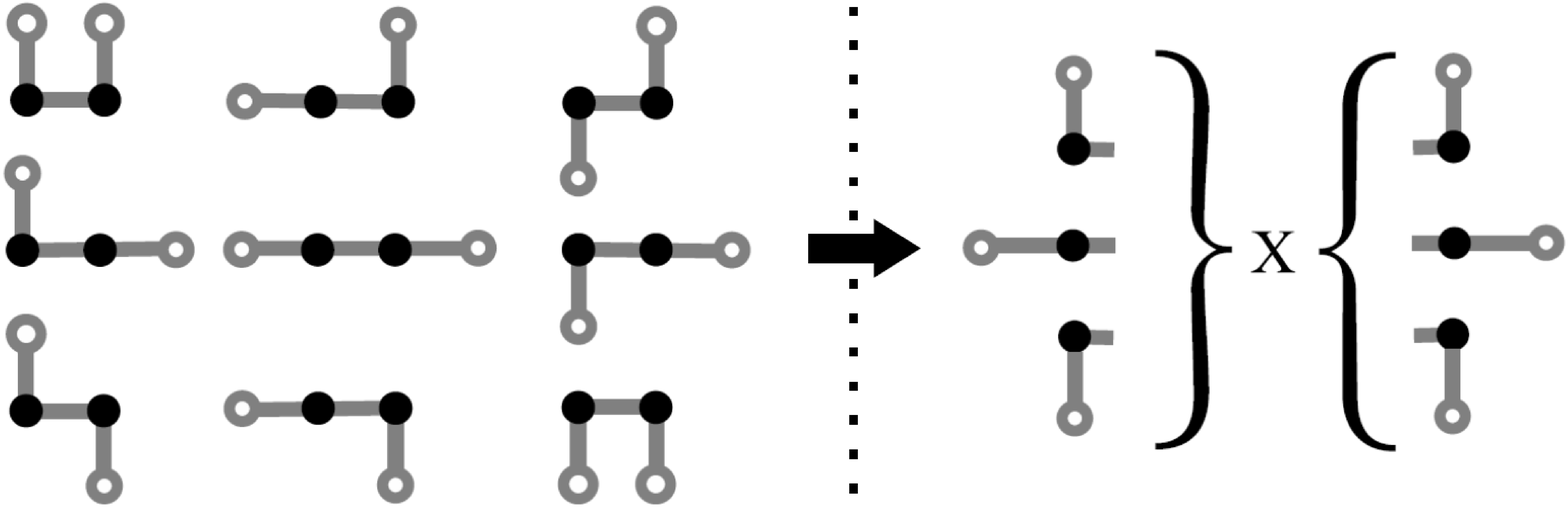}
    \end{minipage}
	}
\end{tabular}
    \caption{(a) Contact energy function (b) Example HP-structure with energy~-2 
    (H-monomer: black, P-monomer: white, structure back bone: grey, HH-contact: 
    dotted) (c) Compression of the solution space from 9 complete down to 
    3~by~3 partial structures.}
    \label{fig:hpenergie}
\end{figure}

\pseudosubsubsec{The constraint model}
In~\cite{Backofen_Will_Constraints2006}, the problem was successfully modeled
as CSP and named Constraint-based Protein Structure Prediction (CPSP). Here,
a variable is introduced for each sequence position and with lattice points as
domains\footnote{In practice, lattice positions are indexed by integers such
  that standard constraint solvers for finite domains over integers are
  applicable.}. The self-avoiding walk is modeled by a sequence of binary
neighboring constraints (ensuring the connectivity of successive monomers) and
a global \textsf{All-different} constraint for self-avoidingness. Supporting
decomposition of the \textsf{All-different} propagator, see Sec.~\ref{sec:global}, is
therefore essential for profiting from DDS.

CPSP uses a database of pre-calculated point sets, called H-cores, that
represent possible optimal distributions of H-monomers. By that, the
optimization problem is reduced to a satisfaction problem for a given H-core,
if H-variables are restricted to these positions.  For optimal H-cores, the
solutions of the CSP are optimal structures. Thus, for counting all optimal
structures, one iterates through the optimal cores.

\pseudosubsubsec{The test sets}
We generated two test sets, PS-48 and PS-64, with uniformly distributed random 
HP-sequences of length~48 and~64. For the generation we used the free available 
CPSP implementation~\cite{CPSP}. With only minimal modifications (new 
branching) we use the existing CSP model with DDS.

PS-48 contains 6350~HP-sequences and for each up to 1~million optimal 
structures in the cubic lattice were predicted. For the 2630~HP-sequences 
in PS-64 up to 2~million structures have been predicted in the cubic 
lattice, due to the increasing degeneracy in sequence length.

\pseudosubsubsec{Results}
The average ratio results are given in Tab.~\ref{tab:res:cpsp}. There, the 
enormous search tree reduction with an average factor of 11~and 25~respectively 
is shown. The reduction using DDS compared to DFS leads to much less 
propagations (3-~to 5-fold). This and the slightly less fails result in a 
runtime speedup of 3-/4-fold using the same variable selection heuristics for 
both search strategies. Here, the immense possibilities of DDS even without 
advanced constraint-graph specific heuristics are demonstrated. This also shows 
the rising advantage of DDS over DFS for increasing problem sizes (with higher
solution numbers).

\begin{table}[hbt]
\centering
{\scriptsize DFS / DDS}\\
        \begin{tabular}{r|c|c|c|c}
          &\ runtime &\ ST size &\ fails &\ propagations \\
         \hline
         PS-48   & 2.98 & 11.30 & 1.40 & 3.27 \\
         PS-64   & 4.23 & 25.33 & 1.76 & 5.43 
        \end{tabular}

\caption{Average ratios for CPSP using DFS vs. DDS (ST = search tree)}
\label{tab:res:cpsp}
\vspace{-0.5em}
\end{table}

\section{Discussion}
\label{sec:disc}
\vspace{-.5em}

The paper introduces \emph{decomposition during search} (DDS), an integration
of \ANDOR{} search with propagation-based constraint solvers. DDS
dynamically decomposes CSPs, avoiding much of the redundant work of standard
tree search when exploring huge search spaces, e.g.\ of $\#P$-hard counting
problems.

We discuss the interaction of DDS with such vital and essential features as
global constraints and dynamic variable ordering. The techniques presented
here have been implemented for Gecode.

The empirical evaluation on graph coloring and protein structure prediction
shows the huge potential of DDS in terms of search tree size reduction and
already high true runtime speedup. The speedup proves that DDS can be
implemented competitively, and with a reasonable overhead. We expect even
higher speedups by improving the constraint graph representation and its
incremental maintenance, which is a current area of development. However, one
experience from our experiments is that it is highly problem-specific whether
the constraint graph allows for decomposition. We partly explain this by
pointing out that some constraints (e.g.\ linear (in-)equations) inherently
hinder decomposition.

We envision promising future research in the following directions. First,
providing efficient access to the constraint graph. Second, the development of
specifically tailored heuristics for DDS focusing on dynamic variable
selection or domain splitting. Such heuristics should employ information about
the constraint graph, to decompose the problem as often as possible and in a
well-balanced way. Decomposition-directed heuristics might however counteract
problem specific heuristics. Balancing such heuristics is a further research
direction.

Finally, solving optimization problems using \ANDOR{} branch-and-bound (BAB)
search~\cite{Dechter:IJCAI:2005} seems an obvious extension. However, our first
experiments using a prototypical DDS extension of BAB show much smaller benefits
than for counting (similar to the results
in~\cite{Dechter:IJCAI:2005,MarinescuD06_ECAI06}).

\pseudosubsubsec{Acknowledgements} We thank Christian Schulte and
Mi\-ka\-el La\-ger\-kvist for fruitful discussions
about the architecture and the paper, and the reviewers of earlier versions
of this paper for constructive comments.
Martin Mann is supported by the EU project EMBIO (EC contract number 012835).
Sebastian Will is partially supported by the EU Network of Excellence REWERSE
(project number 506779).

\bibliographystyle{abbrv}
\bibliography{references}
\end{document}